\documentclass[letterpaper]{article} 
\usepackage{aaai2027}  
\usepackage[hyphens]{url}  
\usepackage{graphicx} 
\usepackage{natbib}  
\usepackage{caption} 
\usepackage{algorithm}
\usepackage{algorithmic}

\usepackage{newfloat}
\usepackage{listings}
\DeclareCaptionStyle{ruled}{labelfont=normalfont,labelsep=colon,strut=off} 
\floatstyle{ruled}
\newfloat{listing}{tb}{lst}{}
\floatname{listing}{Listing}

\usepackage{booktabs}
\usepackage{amsmath}
\usepackage{amssymb}
\usepackage{mathtools}
\usepackage{amsthm}
\usepackage{microtype}
\usepackage{subcaption}
\theoremstyle{plain}
\newtheorem{theorem}{Theorem}[section]

\newtheorem{lemma}[theorem]{Lemma}

\theoremstyle{definition}

\theoremstyle{remark}

\title{Self-Evolving Default Action in Counterfactual Credit Assignment for Cooperative Continuous Control}
\author{
    Shuangyao Huang\textsuperscript{\rm 1}\corresponding\thanks{Supported by XJTLU Research Development Fund: RDF-23-02-026. }
}
\affiliations{
    \textsuperscript{\rm 1}Xi'an Jiaotong-Liverpool University\\
    School of Internet of Things\\
    shuangyao.huang@xjtlu.edu.cn
}

\begin{document}

\maketitle

\begin{abstract}
	Counterfactual credit assignment has proven effective in multi-agent reinforcement learning (MARL) for discrete action spaces, yet its extension to continuous-action cooperative tasks remains challenging. Existing methods that approximate the counterfactual baseline via Monte Carlo sampling often introduce bias into policy gradients and fail to guarantee convergence to local optima, as the sampled actions may not have been sufficiently trained. To address these limitations, we propose SAFE, a novel MARL framework that employs a counterfactual baseline conditioned on a self-evolving default action sampled from each agent's experience buffer. This design naturally extends to continuous action spaces without relying on additional simulations, reward models, or environment-specific prior knowledge. The baseline accurately quantifies each agent's contribution, and introduces no bias into the deterministic policy gradient, ensuring convergence to local optima. Extensive experiments on cooperative vehicular tasks demonstrate that SAFE consistently outperforms state-of-the-art models. 
\end{abstract}


\section{Introduction} 


Multi-agent reinforcement learning (MARL) has achieved notable success in simulated decentralized cooperative games, such as StarCraft II \cite{Starcraft2}, yet its deployment in real-world continuous-control tasks, such as autonomous vehicular networking, remains limited. These applications present three critical challenges. First, these applications usually feature continuous control, precluding credit assignment schemes that are designed around discrete action space. Second, these applications typically require fully decentralized policies for safety and prompt decision-making, since reliable communication is not guaranteed under all circumstances due to environmental limitations. Last, these applications present large joint state and action spaces, requiring credit assignment schemes to be scalable under the centralized training with decentralized execution (CTDE) \cite{ctde} paradigm. 



Decentralized cooperative models that work with continuous action spaces have been largely limited to multi-agent deep deterministic policy gradient (MADDPG) \cite{maddpg} and MAPPO \cite{mappo}, but they do not scale well under the CTDE paradigm, as it requires the joint observation and action of all agents as input. Credit assignment schemes bypass the scaling issues associated with joint observation and action space, through counterfactual baseline or value factorization. The value factorization schemes \cite{vdn, qmix} have limited representation capability as they impose restrictive monotonicity assumptions and do not extend naturally to continuous action space as agents select actions greedily by taking $\arg\max$ on their policy. On the other hand, counterfactual credit assignment (COMA) \cite{coma} offers an alternative by quantifying each agent's contribution with a baseline that marginalizes out the agent's action from a centralized critic. However, COMA is inherently limited to discrete action spaces, and naive Monte Carlo extensions to continuous spaces introduce bias into policy gradients and lack convergence guarantees, because the sampled actions may be undertrained and the baseline expectation is not zero under the joint policy. 


This raises a central research question: \textit{How can we design a counterfactual baseline for continuous action spaces that accurately quantifies each agent's contribution while preserving unbiased policy gradients and convergence guarantees?} The success of counterfactual credit assignment \cite{coma} suggests that the optimal baseline is derived by taking the expectation of the $Q$ value of all available actions. In this way, the baseline is learned directly from agents' experiences, rather than reward models or user selected task-dependent default actions. However, this approach is limited to discrete action spaces, since the expectation presumes a finite set of available actions for each agent. One possible extension to continuous action spaces is to approximate the expectation with Monte Carlo sampling. However, unlike counterfactual values in discrete action spaces, the sampled actions may not have been sufficiently trained to the centralized critic. As a result, the values conditioned on these actions do not accurately quantify the contribution of an agent. Moreover, these undertrained sampled actions introduce biases to the joint policy gradients, since the expectation of the $Q$ values conditioned on the sampled actions is not zero. 

We answer this question by introducing SAFE (\textit{S}elf-evolving default \textit{A}ction \textit{F}rom \textit{E}xperiences), a novel MARL framework whose core innovation is a counterfactual baseline conditioned on a self-evolving default action sampled from each agent's experience replay buffer. Unlike prior work, this baseline naturally extends to continuous action space. As we proved mathematically, this baseline introduces zero bias into the deterministic policy gradient, ensuring convergence to local optima. Moreover, our experiments demonstrate that this baseline adapts dynamically as the agent's behavior evolves during training, with the sampled default action increasingly reflecting the agent's average performance. Empirically, on cooperative autonomous driving tasks with varying numbers of vehicles and obstacles, SAFE consistently outperforms state-of-the-art baselines including VDN, QMIX, and COMA. Ablation studies further confirm that the performance gains stem directly from the credit assignment mechanism, not merely from continuous action adaptation. Parameter analyses reveal that a single sampled default action forms a more effective baseline than averaging multiple samples, validating our design principle. 

\section{Related Work} 

Independent $Q$-learning (IQL) \cite{iql} is a strong benchmark in MARL tasks with discrete action spaces. However, IQL has no convergence guarantees, as agents independently explore and update their policies, the environment becomes non-stationary from each agent's viewpoint. MADDPG \cite{maddpg} exploits the CTDE paradigm that allows agents to access all local information and global state in training, addressing the non-stationary problem. However, it does not scale well with the number of agents as the input dimension grows exponentially. 

Credit assignment schemes provide better solutions for large scale multi-agent tasks. Explicit credit assignment schemes such as counterfactual multi-agent policy gradients (COMA) \cite{coma} and Shapley $Q$-learning \cite{shapley, shapley1} 
explicitly derive an advanced function for each agent that evaluates the contribution of the agent's action in the cooperative task. COMA estimates an agent's contribution with a baseline defined by the expectation of counterfactual $Q$ values. The agent's policy is then updated to maximize the difference rewards \cite{difference_rewards} between the actual $Q$ value and this baseline. However, they show poor performance in practice due to the sample inefficiency of the on-policy training approach of COMA and the high computational complexity in deriving the baseline with Shapley $Q$ values. 

On the other hand, implicit credit assignment schemes such as value-decomposition networks (VDN) \cite{vdn} and QMIX \cite{qmix} factorize the joint action values as a monotonic combination of individual action values based on the individual-global-max (IGM) condition, which guarantees consistency between the joint and individual action values. However, value factorization leads to unbounded divergence of individual value functions and unpredictable agent behaviors \cite{divergence}. Moreover, the representation capability of the mixing network is limited as it strongly depends on the assumption of monotonic relations. Most importantly, existing credit assignment schemes are specifically designed for discrete actions spaces and do not extend well to multi-agent tasks involving continuous action spaces.


\section{Background} 


We consider a fully cooperative multi-agent setting in which a team of agents select sequential actions within a stochastic and partially observable environment. This problem can be formulated as a decentralized partially observable Markov decision process (Dec-POMDP). 
A Dec-POMDP is defined by a tuple $(\mathcal{N}, \mathcal{S}, \mathcal{A}, r, \mathcal{O}, \mathcal{P})$, where $\mathcal{N} \coloneq \{1, \cdots, N\}$
denotes the set of $N$ agents. $\mathcal{S}$, $\mathcal{O}$ and $\mathcal{A}$ are state, observation and action space of the environment, respectively. Lastly, $r$ is a reward signal and $\mathcal{P}$ is the state transition probability of the environment. 

At each step, each agent makes an observation $o_i\in \mathcal{O}$, $i\in\mathcal{N}$ on the environment, and selects an action $a_i\in \mathcal{A}$, $i\in\mathcal{N}$ based on its observation, yielding the joint action $\boldsymbol{a} \coloneq \{a_i\}_{i=1}^{N}\in\mathcal{A}^N$. After all agents execute their actions $\boldsymbol{a}$ in state $s \in \mathcal{S}$ simultaneously, the environment transits from state $s$ to $s'$ with probability $P(s'|s, \boldsymbol{a}): \mathcal{S}\times\mathcal{A}^N\times\mathcal{S}\rightarrow[0,1]$, and returns a numerical reward $r(s, \boldsymbol{a}): \mathcal{S}\times\mathcal{A}^N\rightarrow\mathbb{R}$. 
In addition, the history of each agent's observations and actions is denoted $\tau_i \coloneq \{(o_i, a_i)_t\}_{t=1}^{T}$. 

In fully cooperative tasks, all agents within the team receive one team reward. Eventually, each agent optimizes its policy $a_i = \pi(\tau_i)$ by maximizing a long-term team return defined by the accumulation of the team reward over time, which induces a joint action-value that follows the Bellman equation: 
\begin{equation} \label{joint action-value} 
	\begin{gathered} 
		Q_{tot}(s_t, \boldsymbol{a}_t) = \mathbb{E}[\sum_{l=0}^{\infty}\gamma^{l}r_{t+l}] 
		= r_t + \gamma \mathbb{E}[Q_{tot}(s_{t+1}, \boldsymbol{a}_{t+1})], 
	\end{gathered} 
\end{equation} 
where $\gamma\in[0, 1)$ is a discount factor. 

\subsection{Deterministic Policy Gradient} 

Deterministic policy gradient (DPG) \cite{dpg, ddpg} is an actor-critic method specifically adapted for continuous action space. Similar to $Q$-learning \cite{dqn}, the actions are selected maximizing the optimal action-value function at each step. When the action space is continuous, the action-value function is presumed to be differentiable with respect to the action. Assuming the actions are selected by a deterministic policy $a_t = \pi(s_t)$ which maximizes the $Q$ function, the optimal $Q$ value can be approximated with: 
\begin{equation} \label{dpg} 
	\begin{gathered} 
		\max_{a}Q(s_t, a)\approx Q(s_t, \pi(s_t)), 
	\end{gathered} 
\end{equation} 
where $Q(s_t, \pi(s_t))$ is subject to Eq. \eqref{joint action-value}. 

With this approximation, the critic can be updated with gradient descent by solving: 
\begin{equation} \label{dpg_critic} 
	\begin{gathered} 
		\min_{\phi} \mathbb{E}_{(s, a, r, s')\sim \mathcal{D}} \biggl[\biggl(Q_{\phi}(s, a)-\Bigl(r+\gamma Q_{\bar{\phi}}(s', \pi_{\bar{\theta}}(s'))\Bigr)\biggr)^2\biggr], 
	\end{gathered} 
\end{equation} 
and the actor is optimized with gradient ascent by maximizing the critic's output: 
\begin{equation} \label{dpg_actor} 
	\begin{gathered} 
		\max_{\theta} \mathbb{E}_{s\sim \mathcal{D}} [Q_{\phi} (s, \pi_{\theta}(s))], 
	\end{gathered} 
\end{equation} 
where $\mathcal{D}$ is a reply buffer, $\phi$ and $\theta$ are parameters of the critic and actor, and $\bar{\phi}$ and $\bar{\theta}$ are parameters of the target networks.

\subsection{Counterfactual Credit Assignment} 

Counterfactual credit assignment \cite{coma, shapley} derives a counterfactual baseline from the centralize critic for each agent. An advantage function is derived by subtracting the counterfactual baseline from the joint action value following the concept of difference rewards: 
\begin{equation} \label{coma_advantage} 
	\begin{gathered} 
		A_i(s, \boldsymbol{a}) = Q(s, \boldsymbol{a}) - b_i, 
	\end{gathered} 
\end{equation} 
where $i\in\mathcal{N}$. The counterfactual baseline $b_i$ is derived by marginalizing out an agent or its action from the centralized critic, and replacing with another agent or default actions. In COMA, $b_i$ is obtained by replacing an agent's action with every possible individual action available to that agent, and then taking the expectation of all resulting critic values: 
\begin{equation} \label{coma_baseline} 
	\begin{gathered} 
		b_i = \sum_{\bar{a}_i\in \mathcal{A}} \pi(s, \bar{a}_i)Q(s, \boldsymbol{a}_{-i}, \bar{a}_i), 
	\end{gathered} 
\end{equation} 
where $\pi(s, a)$ is a stochastic policy defined for discrete action space and $\boldsymbol{a}_{-i}$ denotes the joint action except the action taken by the current agent. 

Lastly, the policy gradient is given by: 
\begin{equation} \label{coma_pg} 
	\begin{gathered} 
		g_{\theta} = \mathbb{E}_{\pi} \biggl[\sum_{i\in\mathcal{N}}\nabla_{\theta}\log \pi_{i}(s, a_i) A_i(s, \boldsymbol{a})\biggr], 
	\end{gathered} 
\end{equation} 
with ${\theta}$ being the parameters of actors. 

Given the fact that $Q(s, \boldsymbol{a})$ evaluates the global $Q$-value of the taken action and the baseline reflects the average result that can be obtained under all other possible actions of the agent, the advantage function in Eq. \eqref{coma_advantage}, therefore, reflects the advantage of the taken action over the average result. 

It is worth noting that the effectiveness of the baseline in Eq. \eqref{coma_baseline} is built on the assumption of a discrete action space $\mathcal{A}$. With this assumption, each action $\bar{a}_i\in \mathcal{A}$ has been executed a sufficient number of times by agent $i$ for state $s$ during training. As a result, the estimation of counterfactual $Q$ value $Q(s, \boldsymbol{a}_{-i}, \bar{a}_i)$ is sufficiently trained. Consequently, the baseline $b_i$ accurately quantifies the contribution of agent $i$. 

\section{{Algorithm}}
%

\begin{figure*}[bt] 
	\centering
	\begin{subfigure}[b]{.60\linewidth}
		\centering
		\includegraphics[width=1.0\linewidth]{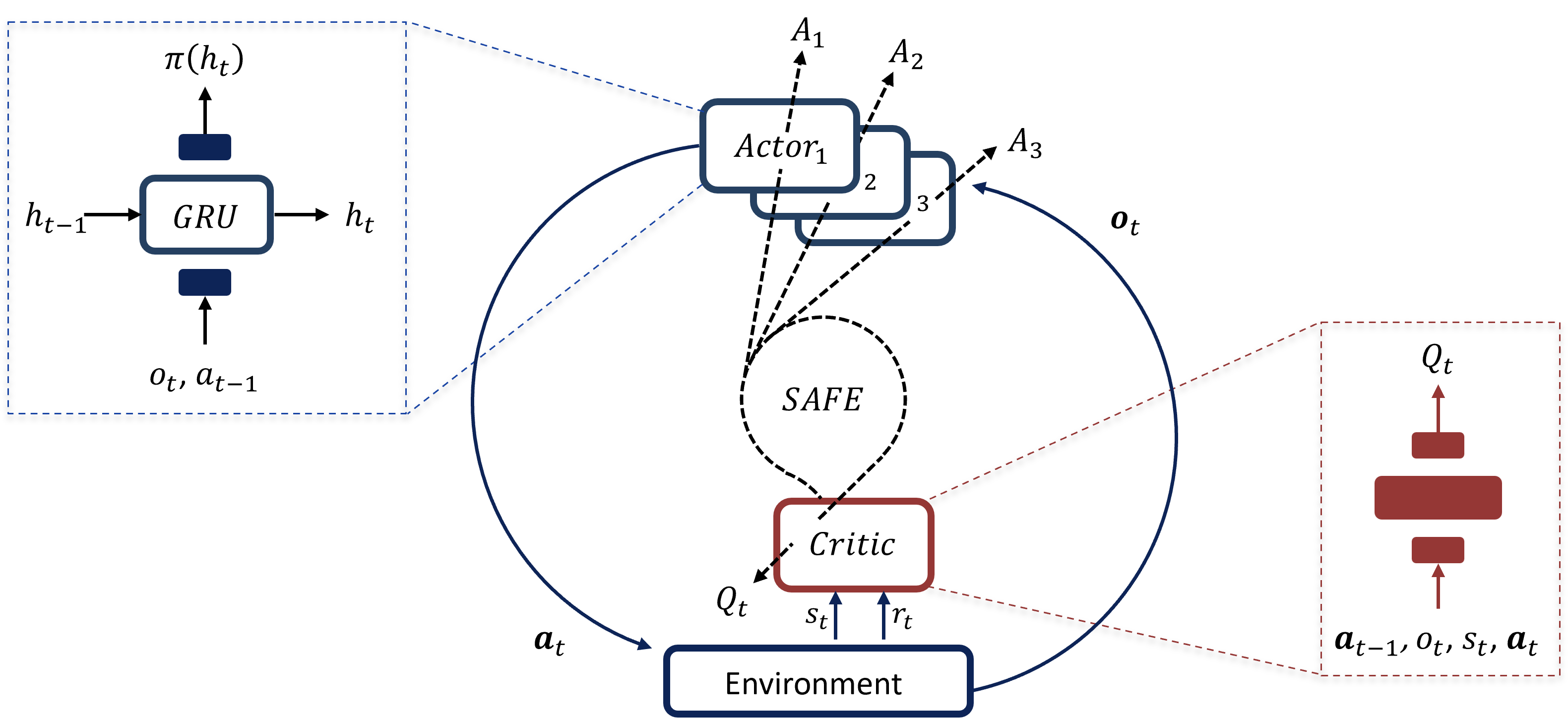} 
		\caption{} 
		\label{framework1} 
	\end{subfigure} 
	\begin{subfigure}[b]{1.0\linewidth}
		\centering
		\includegraphics[width=1.0\linewidth]{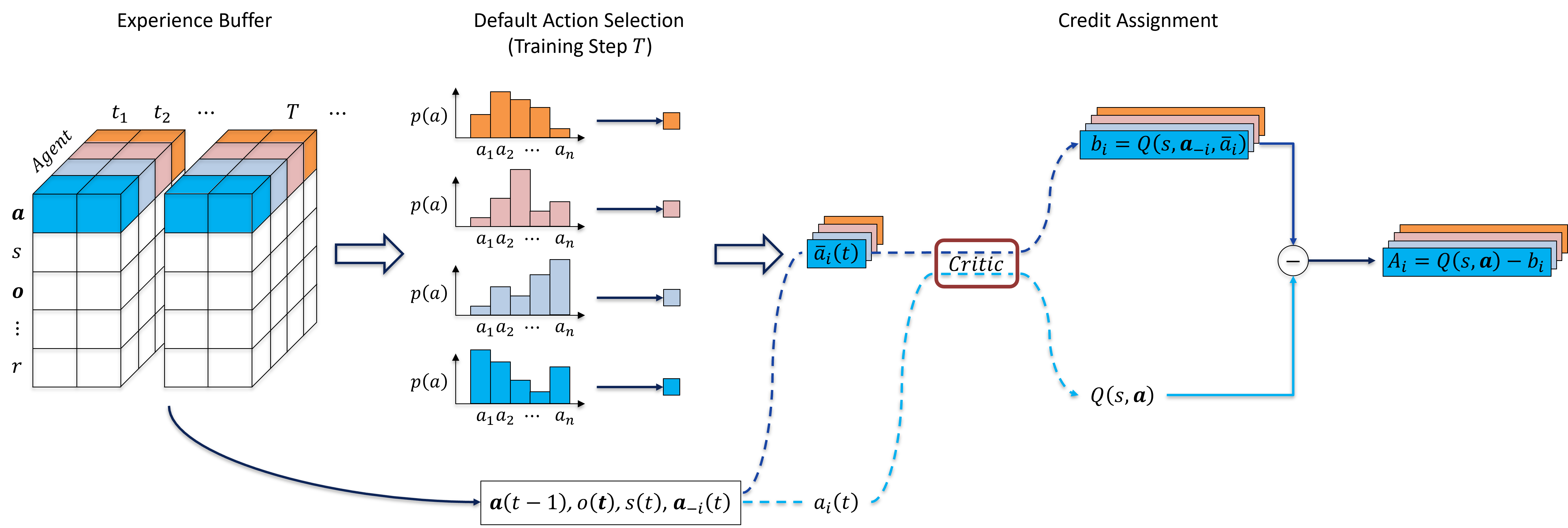} 
		\caption{} 
		\label{framework2} 
	\end{subfigure} 
	\caption{(a). Architectures of the centralized actors and the centralized critic, as while as the information flow between them for credit assignment in training. (b). Credit assignment with the self-evolving default action selected from the agents' experience buffer. } 
	\label{fig:framework} 
\end{figure*} 

Counterfactual credit assignment has gain considerable success in multi-agent settings with discrete action spaces. However, it leads to poor performances in continuous-action tasks, especially for real-world tasks that require fine-grained control where a subtle change in action can lead to drastically different outcomes. Approximating the expectation in Eq. \eqref{coma_baseline} with Monte Carlo sampling does not guarantee performance as the sampled actions may not have been frequently executed during training. Moreover, the baseline approximated via Monte Carlo sampling introduces bias into the policy gradient and does not guarantee convergence to local optima of policies. 
This highlights the importance of developing a counterfactual baseline that quantifies the contribution of continuous actions. The remaining part of this section is dedicated to addressing the above difficulty. 

First, this paper adopts a centralized critic conditioned on the true global state $s$ and the joint action-observation histories $\boldsymbol{\tau}$. The centralized critic is feasible as it is used only during training and only the actor is needed for execution. Each actor conditions on its own action-observation histories $\tau_i$. Parameter sharing is used among actors as agents have the same role in most cooperative applications, such as vehicular control. The network architectures are illustrated in Fig. \ref{framework1}. 

For credit assignment, this paper adopts a counterfactual baseline inspired by difference rewards, in which each agent learns from a shaped reward: 
\begin{equation} \label{diff_reward} 
	\begin{gathered} 
		D_i = r(s, \boldsymbol{a}) - r(s, (\boldsymbol{a}_{-i}, \bar{a}_i)). 
	\end{gathered} 
\end{equation} 

Eq. \eqref{diff_reward} compares the global true reward received after each step of an MDP process and a reward conditioned on a default action $\bar{a}_i$ of agent $i$. Any action $a_i$ that improves $D_i$ also improves the team reward $r(s, \boldsymbol{a})$ as $r(s, (\boldsymbol{a}_{-i}, \bar{a}_i))$ does not depend on $a_i$. However, Eq. \eqref{diff_reward} requires repeated access to a training simulator to estimate the value of actions $\boldsymbol{a}$ and $(\boldsymbol{a}_{-i}, \bar{a}_i)$ for the same given state $s$. This is infeasible when the state space is continuous as it dramatically increases the number of simulations. Moreover, it requires a carefully selected user-defined default action, which relies on in-depth knowledge of the task and must be modified on a per-task basis. Additionally, the default action should be selected such that the expectation of critic values with regard to the policy conditioned on the default action is zero, bringing no bias to the policy gradient. COMA \cite{coma} implements difference rewards using a centralized critic, which avoids the difficulties associated with simulation overload and default action selection while ensuring policy convergence. However, its critic architecture supports only discrete action spaces. 

A key idea of this paper is that difference rewards can be implemented using a centralized critic with a default action that is selected from the agent's experience. Specifically, the default action is chosen based on the probability density distribution of all previously executed actions and evolves adaptively as the agent learns from the environment. 
In this way, the approximation of difference rewards can be extended to continuous action space while remaining applicable to discrete action space. Moreover, the selection of default action does not rely on additional simulations, reward models, or environment dependent knowledge. 

This paper trains a centralized critic that estimates the global $Q$ value for the joint action $\boldsymbol{a}$ under a state $s$. For each agent $i$, a counterfactual baseline is derived from the centralized by replacing the $a_i$ with a default action $\bar{a}_i(t)$ at training step $t$, while maintaining other agents' inputs unchanged: 
\begin{equation} \label{our_baseline} 
	\begin{aligned} 
		b_i = Q(s, \boldsymbol{a}_{-i}, \bar{a}_i(t)), \ \bar{a}_i(t)\sim \mathcal{D}_t. 
	\end{aligned} 
\end{equation} 

The default action $\bar{a}_i(t)$ is sampled uniformly from the agent's experience buffer $\mathcal{D}_t$ at training step $t$. Consequently, $b_i$ estimates the average performance of agent $i$ with increasing accuracy, as $\bar{a}_i(t)$ approaches the action that is most frequently executed by the agent as $t$ increases. As a result, the values conditioned on these actions will be sufficiently trained and the advantage function $A_i$ in Eq. \eqref{coma_advantage} accurately quantify the contribution of agent $i$. 
The credit assignment with the self-evolving default action selected from the agents' experience buffer is illustrated in Fig. \ref{framework2}.

Next, we demonstrate that a policy gradient using our baseline converges to a local optimum in \textbf{Lemma} \ref{lemma1}. In other words, the baseline does not affect the convergence of the deterministic policy gradient. Previous work on single agent actor-critic \cite{suttonpg} has shown that a baseline $b$ does not affect the convergence of stochastic policy gradient if $b$ is action-independent. COMA \cite{coma} further proves that in multi-agent scenarios, a stochastic policy gradient is not affected if baseline $b_i$ is $a_i$ independent. A natural question arises here: \textit{Can this conclusion be generalized to a deterministic policy gradient?} The authors in \cite{dpg} have illustrated that deterministic policy $\pi_{\theta}$ is a special case of stochastic policy by writing the stochastic policy as $\mu_{\pi_{\theta}, \sigma}$, where $\sigma$ represents a variance, such that 
\begin{equation} \label{conversion} 
	\begin{gathered} 
		\lim_{\sigma\rightarrow0}\mu_{\pi_{\theta}, \sigma}=\pi_{\theta}.
	\end{gathered} 
\end{equation}

%
%

Let set $\mathcal{A}_{\sigma} = \{a_1, a_2, \cdots, a_n\}$ denote a subset of $\mathcal{A}$: $a_i\in\mathcal{A}$, for $i\in [1, n]$. By the definition of stochastic policy, $\sum_{a\in\mathcal{A}_{\sigma}}\mu_{\pi_{\theta}, \sigma}(a|\tau)=1$. We have \textbf{Lemma} \ref{lemma1} that demonstrates our baseline ensures convergence of stochastic policy gradients and generalizes to deterministic policy gradients. 
\begin{lemma} \label{lemma1}
	For a policy gradient at step $t$: 
	\begin{equation} \label{proof1} 
		\begin{gathered} 
			g_t = \mathbb{E}_{\mu} \left[\sum_{i}\nabla_{\theta}\log\mu_{\theta, \sigma}(a_i|\tau_i)\cdot A_i
			\right], \\ 
			A_{i} = Q(s, \boldsymbol{a})-b_i(s, \boldsymbol{a}_{-i}). 
		\end{gathered} 
	\end{equation}
	where $b_i$ is an agent-specific baseline defined in Eq. \eqref{our_baseline}, we have 
	\begin{equation} \label{proof2} 
		\begin{gathered} 
			\lim\limits_{t\rightarrow \infty} ||g_t|| = 0. 
		\end{gathered} 
	\end{equation}
\end{lemma}

\begin{proof} 
	We need to only look at the expectation of the baseline $b_i$ with regard to the current policy. Since $\mu_{\theta, \sigma}(\cdot|\tau)$ is a stochastic policy, we can write the joint policy of all agents as $\boldsymbol{\mu}(\boldsymbol{a}|s) = \prod_{i}\mu_i(a_i|\tau_i)$. Let $d^{\boldsymbol{\mu}}(s)$ be the discounted state distribution probability: $d^{\boldsymbol{\mu}}(s)= \sum_{t=0}^{\infty}\gamma^tPr(s_t=s|s_0, \boldsymbol{\mu})$, the expectation of the baseline is written as: 
	\begin{equation*} \label{gb} \nonumber
		\begin{aligned} 
			g^{(1)} &= -\mathbb{E}_{\mu} \left[\sum_{i}\nabla_{\theta}\log\mu_i(a_i|\tau_i)\cdot b_i\right] \\ 
			= -&\sum_{s}d^{\boldsymbol{\mu}}(s) \sum_{i} \sum_{\boldsymbol{a}^{-i}}\boldsymbol{\mu}^{-i}\cdot  
			{\sum_{{a}_{i}\in\mathcal{A}_{\sigma}} \mu_i({a}_i|\tau_i) \nabla_{\theta}\log\mu_i({a}_i|\tau_i) b_i} \\ 
			= -&\sum_{s}d^{\boldsymbol{\mu}}(s) \sum_{i} \sum_{\boldsymbol{a}^{-i}}\boldsymbol{\mu}^{-i}\cdot 
			\underbrace{b_i \sum_{{a}_{i}\in\mathcal{A}_{\sigma}} \mu_i({a}_i|\tau_i) \nabla_{\theta}\log\mu_i({a}_i|\tau_i)}_{(b_i \text{ is not a function of }{a}_i)} \\
			= -&\sum_{s}d^{\boldsymbol{\mu}}(s) \sum_{i} \sum_{\boldsymbol{a}^{-i}}\boldsymbol{\mu}^{-i}\cdot 
			b_i \nabla_{\theta}\sum_{{a}_{i}\in\mathcal{A}_{\sigma}}\mu_i({a}_i|\tau_i)\\ 
			= -&\sum_{s}d^{\boldsymbol{\mu}}(s) \sum_{i} \sum_{\boldsymbol{a}^{-i}}\boldsymbol{\mu}^{-i}\cdot 
			b_i \nabla_{\theta}1\\
			=\ \ \ \ &0. 
		\end{aligned} 
	\end{equation*}
	
	The reminder of Eq. \eqref{proof1} is simply the standard single-agent actor-critic policy gradient: 
	\begin{equation} \label{proof3} 
		\begin{gathered} 
			g_t^{(2)} = \mathbb{E}_{\mu} \left[\nabla_{\theta}\log \boldsymbol{\mu}(\boldsymbol{a}|s) Q(s, \boldsymbol{a}) \right]. 
		\end{gathered} 
	\end{equation}
	
	It is shown in \cite{actor-critic} that an actor-critic following this gradient converges to a local maximum of the objective function: 
	\begin{equation} \label{proof4} 
		\begin{gathered} 
			\lim\limits_{t\rightarrow \infty} ||g_t^{(2)}|| = 0. 
		\end{gathered} 
	\end{equation}
	
	Hence, $b_i$ does not introduce bias in the policy gradient of $\mu$, and therefore $\pi$, as $\mu\equiv\pi$ when $\sigma=0$. 
\end{proof}

\section{Experimental Setup}

Highway-Env \cite{highway-env} is a \textit{Gym} compatible environment designed for cooperative vehicular tasks, including highway, intersection, and merging traffic scenarios. We focus on cooperative autonomous driving in highway scenario which is illustrated in Fig. \ref{fig:scenario}, where the controlled vehicles (green) are driving on a multi-lane highway populated with obstacle vehicles (blue). In this task, the controlled vehicles travel from the left of the monitor to the right, while the obstacle vehicles travel in the opposite direction toward them. The controlled vehicles are positioned in close proximity to each other, such that even a minor change in heading may result in a collision. 

\begin{figure}[bt]
	\centering
	\includegraphics[width=0.45\textwidth]{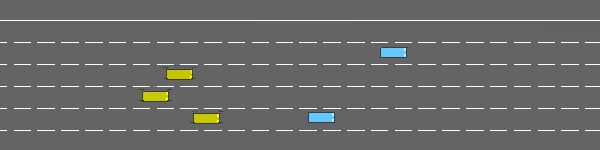} 
	\caption{Highway cooperative autonomous driving tasks. Controlled vehicles are shown in green, while obstacle vehicles are depicted with blue. }
	\label{fig:scenario}
\end{figure}

During training, each agent observes a set of kinematic features for itself and its neighbors, including position, velocity, and heading. The action of an agent directly controls its steering angle in a range of $[-{\pi}/{4}, {\pi}/{4}]$ to maintain proximity to neighbors. The reward function for an agent includes collision penalty, and components encouraging formation and efficiency. The team reward is the sum of all agent rewards. More details can be found in the environment's documentation. 


Experiments are conducted under four scenarios: 2-Vehicle-1-Obstacle (2V1O), 3-Vehicle-2-Obstacle (3V2O), 5-Vehicle-2-Obstacle (5V2O), and 7-Vehicle-2-Obstacle (7V2O). The benchmark methods compared are VDN, QMIX, IQL, COMA, and MAPPO. 
For models designed for discrete action spaces, the agents use discrete meta-actions, namely \{0: \textit{LEFT LANE}, 1: \textit{IDLE}, 2: \textit{RIGHT LANE}, 3: \textit{FASTER}, 4: \textit{SLOWER}\}. All experiments are conducted with three independent runs on an NVIDIA RTX 5060Ti GPU with 16$GB$ memory. 

\subsection{Network Structure and Training}

The actor consists of 128-neuron gated recurrent units (GRU) \cite{gru} that use fully connected layers to process the input and to produce the output values from the hidden state. The agent's observation is fed into the input layer. Actions are produced from the final layer $\boldsymbol{z}$ via a $\tanh$ activation added by a random value $\mu$ controlled by parameter $\epsilon$: $a_i=(1-\epsilon)\tanh(\boldsymbol{z})_i + \epsilon \mu$, where $\mu$ is sampled from the action space uniformly. Parameter $\mu$ is annealed exponentially from 1.0 to 0.05 across 50,000 training episodes. The centralized critic consists of multiple fully connected layers with ReLU activation functions. The global state, the agent's observation, and the all agents' joint actions are fed into the input layer. 
The output layer generates the global action value $Q$. 



\subsection{Parameter Study} 

In this experiment, we examine impact of the number of sampled actions in Eq. \eqref{our_baseline}. In other words, we try to answer the question \textit{Which one serves a better baseline, a single default action or the expectation of a group actions}? Mathematically, we examine the impact of $K$ for agent $i$ in:
\begin{equation} \label{parameter} 
	\begin{aligned} 
		b_K = \mathbb{E}_{k=1}^{K} Q(s, \boldsymbol{a}_{-i}, \bar{a}_i^k), \ \bar{a}_i^k\sim \mathcal{D}. 
	\end{aligned} 
\end{equation}

The key idea of SAFE is to estimates the average performance of an agent through a counterfactual baseline, which is derived from the centralized critic by replacing the agent's action with another action sampled from the agent's experience buffer. As the training goes on, the sampled action approaches the action that is most frequently executed by the agent. Therefore, the baseline gradually approximates the average performance of that agent in the task. To answer the above question, we compare SAFE$_{1}$ SAFE$_{30}$, SAFE$_{50}$, SAFE$_{80}$, and SAFE$_{100}$, where SAFE$_K$ is deviced by baseline $b_K$ and SAFE$_{1}$ is the default setting of SAFE. 

\begin{figure*}[bhtp]
	\centering
	\includegraphics[width=0.7\textwidth]{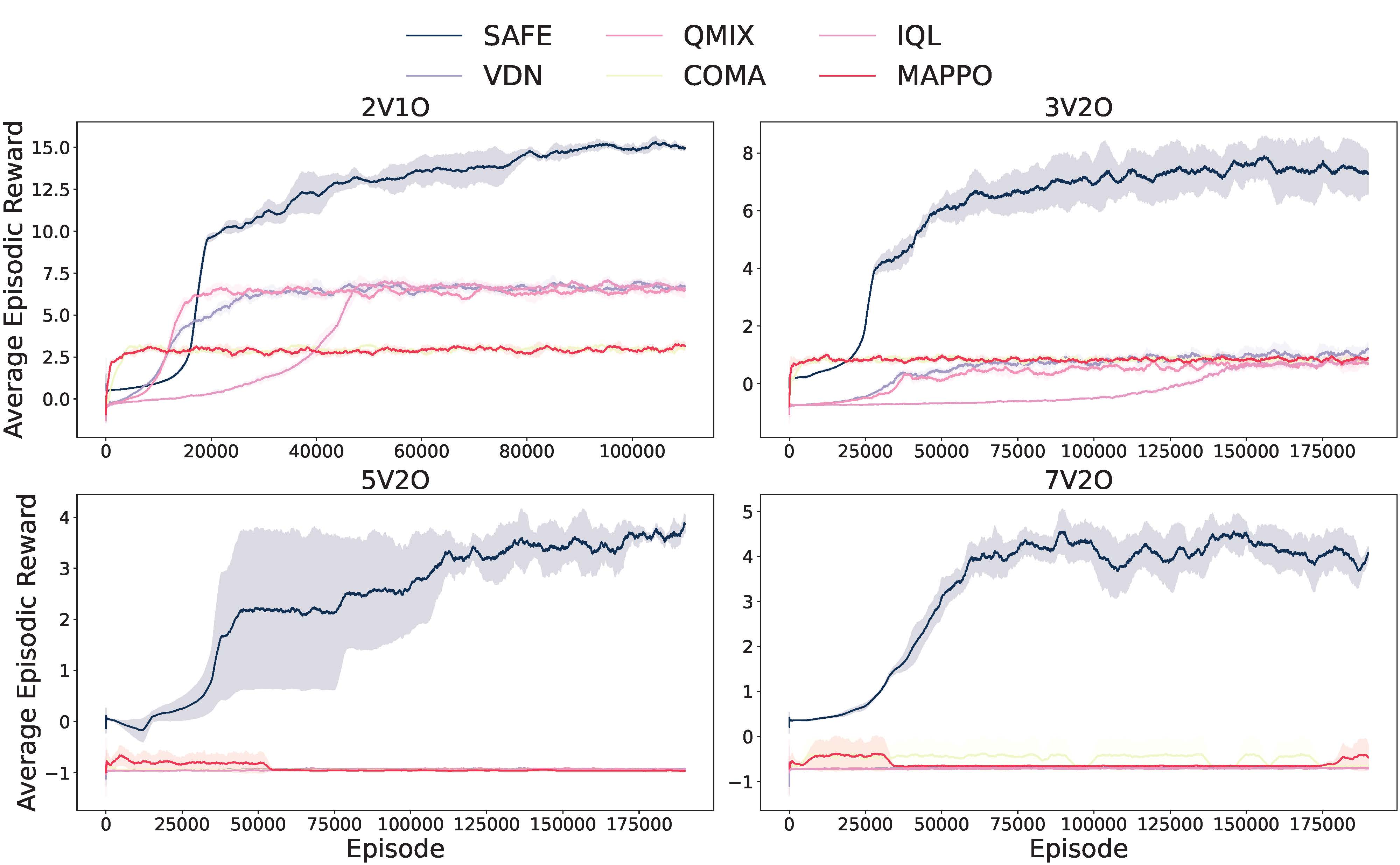} 
	\caption{Comparison of SAFE against benchmark models on highway task, under four scenarios: 2-Vehicle-1-Obstacle (2V1O), 3-Vehicle-2-Obstacle (3V2O), 5-Vehicle-2-Obstacle (5V2O), and 7-Vehicle-2-Obstacle (7V2O). }
	\label{fig:result1}
\end{figure*}

\begin{table*}[thb]
	\centering
	\begin{tabular}{|c |c c c c c c|} 
		\specialrule{.1em}{.1em}{.1em}
		\specialrule{.1em}{.1em}{.1em}
		& SAFE & VDN & QMIX & COMA & IQL & MAPPO \\ [0.5ex]  
		\specialrule{.1em}{.1em}{.1em} 
		2V1O & 0.029 & 0.053 & 0.048 & 0.859 & 0.062 & 0.770 \\ [0.5ex]  
		3V2O & 0.067 & 0.786 & 0.759 & 0.993 & 0.814 & 0.998 \\ [0.5ex] 
		5V2O & 0.082 & 0.982 & 0.981 & 1.000 & 0.990 & 1.000 \\ [0.5ex]  
		7V2O & 0.091 & 0.502 & 0.500 & 0.423 & 0.502 & 0.293 \\ [0.5ex]   
		\specialrule{.1em}{.1em}{.1em} 
		\specialrule{.1em}{.1em}{.1em}
	\end{tabular}
	\caption{Collision rate of trained models across all scenarios. }
	\label{table1}
\end{table*}

\subsection{Ablations} 

We conduct a number of ablations to show the effectiveness of the counterfactual baseline scheme in SAFE. 
In other words, we try to answer the question \textit{Is the superior performance of SAFE achieved by its credit assignment scheme, or its adaptation to continuous action space}? To answer this question, we compare SAFE$_{1}$, the default setting, with the following baselines that also adapt to continuous action space but either adopt a different credit assignment scheme or a centralized critic: 
\begin{itemize}
	\item SAFE\_a=0: Adopts the same baseline as in Eq. \eqref{our_baseline}, except that the replacing action is set to a default action $a_i=0$, rather than sampled from experience buffer. 
	\item SAFE\_batch\_mean: Adopts the same baseline as in Eq. \eqref{our_baseline}, except that the replacing action is set to the mean of a batch of 32 actions sampled uniformly from the experience buffer: $\bar{a}_i = \mathbb{E}_{k=1}^{32} \bar{a}_i^k, \ \bar{a}_i^k\sim \mathcal{D}$. 
	\item COMA\_cont: Continuous COMA, which adapts to continuous action space by approximating the baseline with Monte Carlo sampling.
	\item Centralized\_critic: Multiple agent maintain a centralized critic and decentralized actors. The centralized critic is directly trained by the team reward, while the decentralized actors are updated by maximizing the centralized $Q$. 
\end{itemize}


\section{Experimental Results} 

%



\begin{figure*}[t] 
	\centering
	\begin{subfigure}[b]{0.80\linewidth}
		\centering
		\includegraphics[width=1.0\linewidth]{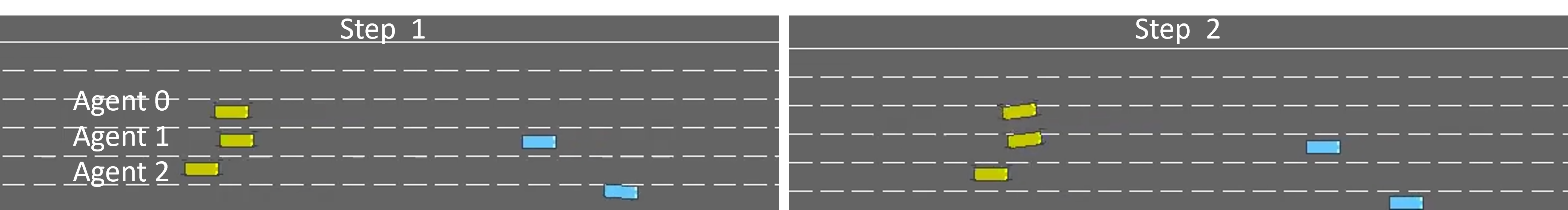} 
		\caption{} 
		\label{fig:result5a} 
	\end{subfigure} 
	\begin{subfigure}[b]{1.0\linewidth}
		\centering
		\includegraphics[width=1.0\linewidth]{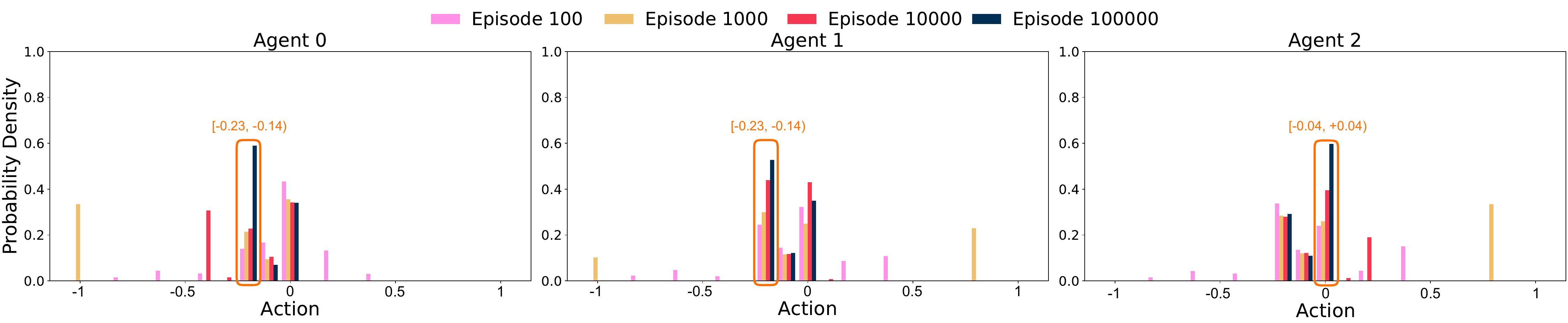} 
		\caption{} 
		\label{fig:result5b} 
	\end{subfigure} 
	\caption{(a). Trajectories of vehicles after 150,000 training episodes. (b). Probability density distributions of each agent's default actions at 100, 1,000, 10,000, and 100,000 training episodes. } 
	\label{fig:results5} 
\end{figure*} 

\begin{figure*}[t] 
	\centering
	\begin{subfigure}[b]{0.45\linewidth}
		\centering
		\includegraphics[width=1.0\linewidth]{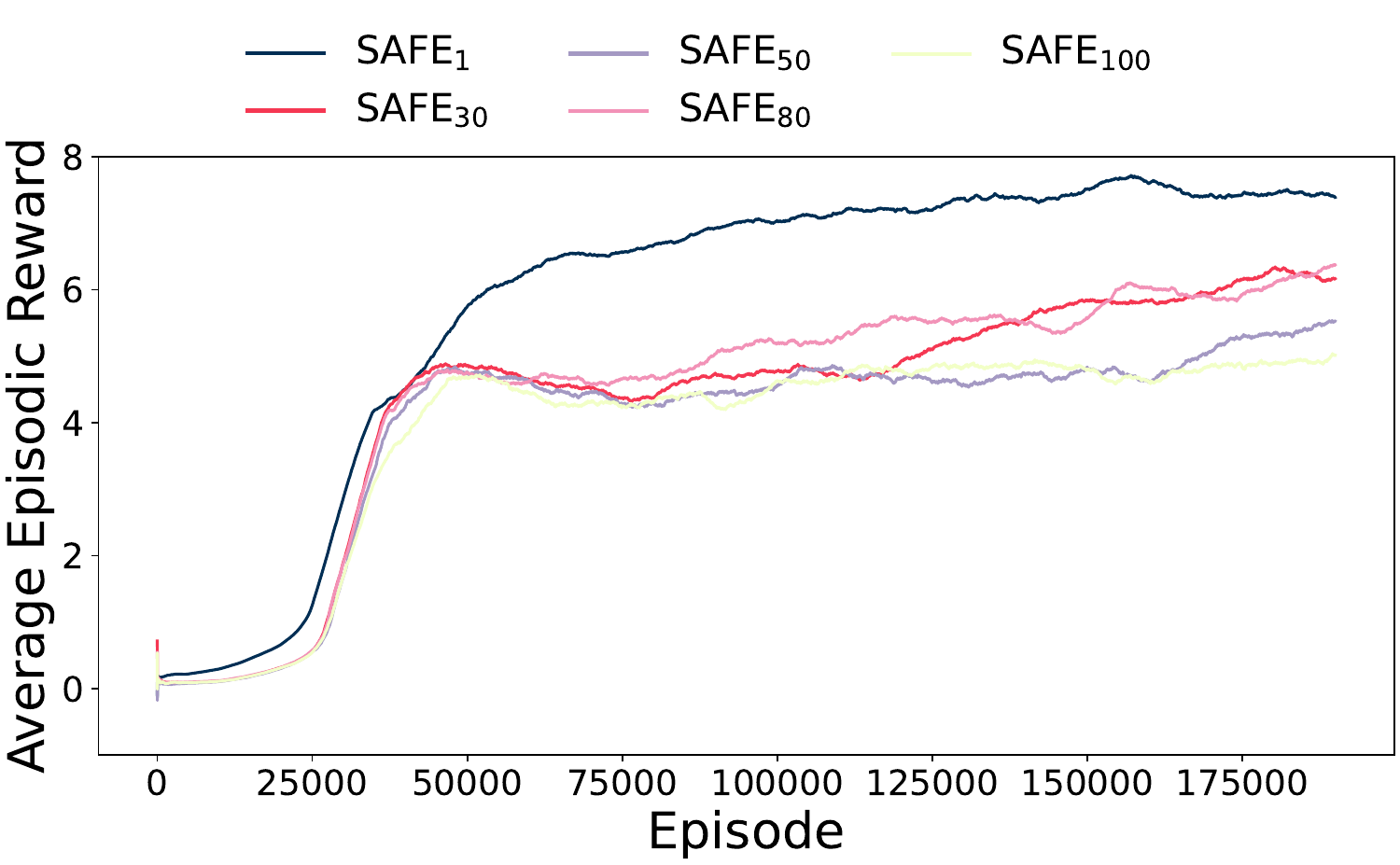} 
		\caption{} 
		\label{fig:result2} 
	\end{subfigure} 
	\begin{subfigure}[b]{0.45\linewidth}
		\centering
		\includegraphics[width=1.0\linewidth]{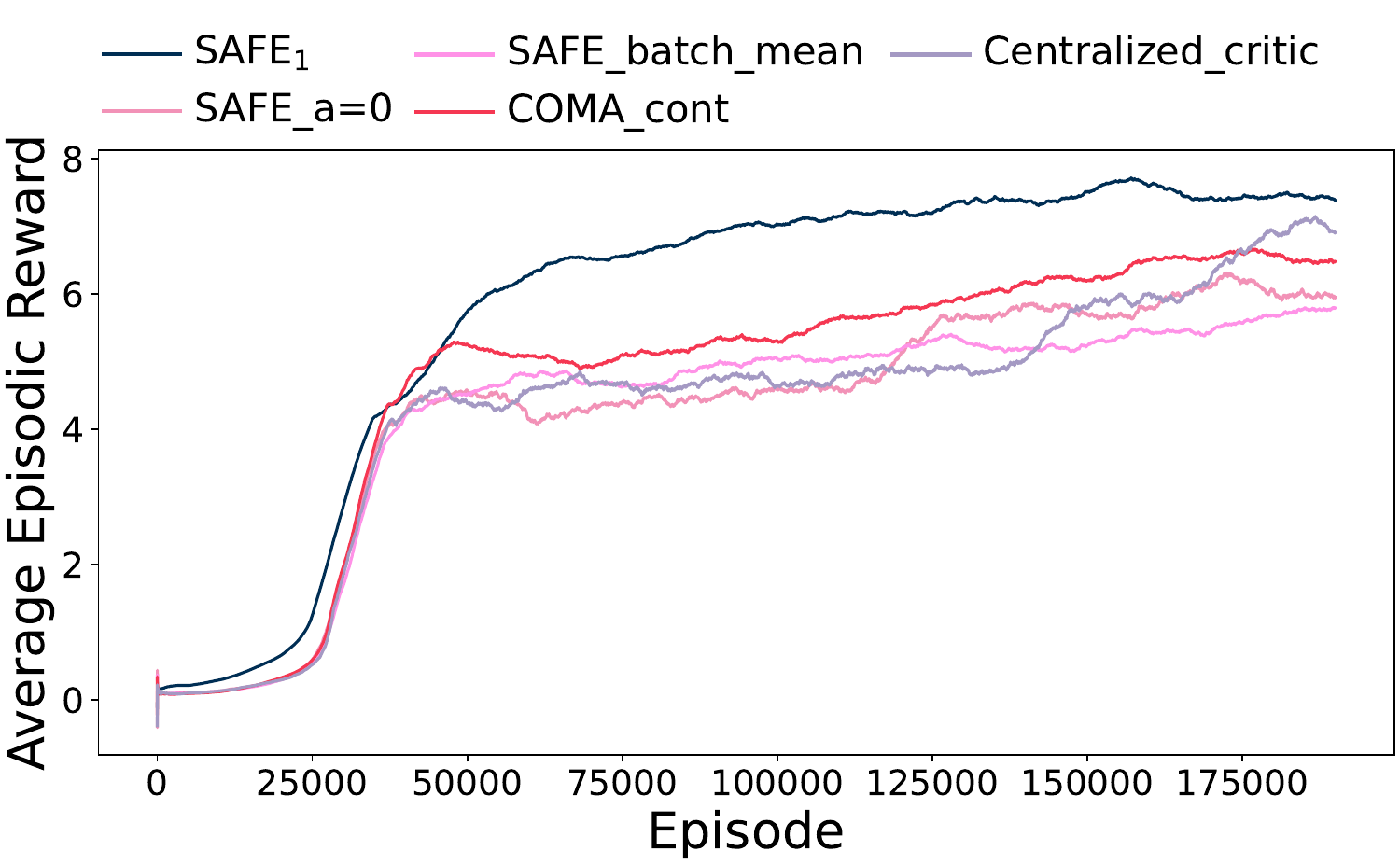} 
		\caption{} 
		\label{fig:result3} 
	\end{subfigure} 
	\caption{(a). Parameter study on the number of sampled actions in constructing our baseline. (b). Ablation study on the effectiveness of our counterfactual credit assignment scheme. } 
	\label{fig:results23} 
\end{figure*} 

Fig. \ref{fig:result1} shows the performance of SAFE in the four scenarios in comparison against the benchmark models. Because agents spawn in close proximity, even minor actions can cause collisions. Thus, the task demands highly precise and sensitive adjustments to succeed. Consequently, all benchmarks with discrete action spaces fail to completely avoid collisions, achieving poor performance across all five scenarios. In contrast, SAFE achieves outstanding performance across all scenarios, owing to its compatibility with continuous action spaces. Only in scenario 2V1O (the simplest of the four) did VDN, QMIX, and IQL converge to stable policies. In this case, the two controlled vehicles can avoid collisions by simply executing the discrete actions \textit{LEFT LANE} and \textit{RIGHT LANE}. However, these models fail in all other scenarios, where more subtle and finely tuned actions are required to prevent collisions. The coarse granularity of discrete meta-actions quickly traps the policies in local minima. It is also worth noting that COMA and MAPPO fail to converge to optimal policies across all four scenarios, largely due to the sample inefficiency of on-policy training. 
The numerical results in Table \ref{table1} demonstrate that SAFE consistently achieves significantly lower collision rates than COMA and QMIX, two state-of-the-art policy-based and value-based MARL models, across all tested scenarios. Collisions in these experiments may occur either between controlled vehicles or between a controlled vehicle and an obstacle vehicle. In scenario 2v1o, SAFE achieves a collision rate of 0.029, which is 97\% lower than COMA and 40\% lower than QMIX. In the more complex scenario 3v2o, SAFE achieves a collision rate of 0.067, representing improvements of 93\% and 91\% over to COMA and QMIX, respectively. This trend continues in scenario 5v2o, where SAFE outperforms both baselines by 92\%. Even in the most challenging scenario 7v2o, SAFE achieves a collision rate which is 78\% lower than COMA and 82\% lower than QMIX. It is worth noting that in scenarios 3v2o, 5v2o, and 7v2o, all the benchmark models fail to converge to locally optimal policies other than SAFE as illustrated in Fig. \ref{fig:result1}. Their collision rates, therefore, do not indicate the quality of policies. For instance, COMA achieving a lower collision rate in scenario 7v2o than in 5v2o is merely a result of the randomness in action selection and state transition. 


The self-evolving default action in our counterfactual baseline can be further examined through Fig. \ref{fig:results5}, which presents the probability density of default actions sampled from each agent's experience replay buffer. During the first 100 training episodes, agents 0 and 1 most frequently sample actions from the bin $[-0.04, +0.04)$. However, after 100,000 episodes, their sampled actions evolves to the bin $[-0.23,-0.14)$. In contrast, agent 2 initially concentrates its sampled actions in $[-0.23, -0.14)$ within the first 100 episodes, but evolves to $[-0.04, +0.04)$ after 100,000 episodes. The rationale behind this evolution is as follows: as training progresses, agents 0 and 1 gradually learn to steer leftward (taking actions $a\in[-0.23, -0.14)$) to avoid obstacle vehicles, while agent 2 tends to maintain its course (taking actions $a\in[-0.04, +0.04)$) to cooperatively avoid obstacles, as illustrated in Fig. \ref{fig:result5a}. Consequently, actions in $a\in[-0.23, -0.14)$ become dominant in the replay buffers of agents 0 and 1, whereas actions in $a\in[-0.04, +0.04)$ come to dominate agent 2's buffer. The default action sampled from the replay buffer effectively reflects the agents' average performance, as it corresponds to the most frequently executed action. In this regard, the distribution shown in Fig. \ref{fig:result5b} is consistent with a stochastic policy for a discrete action space.

Additionally, Fig. \ref{fig:result2} demonstrates that \textit{a single default action sampled from the agent's experience buffer forms a better baseline compared to the expectation of a group actions}. Increasing the number of samples does not help with the estimation accuracy of the baseline in Eq. \ref{our_baseline}, and introduces additional biases yielding unsatisfactory results. 
On the other hand, Fig. \ref{fig:result3} further demonstrates that \textit{the superior performance of SAFE is achieved by its credit assignment scheme, rather than its adaptation to continuous action space}. While all the baseline models adapt to continuous action space of agents, only SAFE yields superior performance, owning to the accuracy of its credit assignment scheme. The experiments are conducted in scenario 3V2O, because the task is neither too simple nor too complex, making it suited for studying the subtle impact of parameters and baselines on model performance. 

\section{Conclusive Remarks} 

This paper presents SAFE, a novel MARL framework for cooperative tasks with continuous action spaces. The key innovation is a counterfactual baseline conditioned on a self-evolving default action sampled from each agent's experience buffer, which naturally extends credit assignment to continuous action space without additional simulations or task-specific priors. We prove theoretically that this baseline introduces no bias into the deterministic policy gradient, guaranteeing convergence to local optima. Experiments on cooperative autonomous driving tasks demonstrate that SAFE consistently outperforms state-of-the-art methods including VDN, QMIX, IQL, COMA, and MAPPO across scenarios of varying complexity. Ablation studies confirm that the performance gains stem from the credit assignment scheme itself rather than mere adaptation to continuous actions. Parameter analysis shows that a single sampled default action forms a better baseline than averaging multiple samples. These results establish SAFE as an effective solution for real-world multi-agent systems requiring continuous control. 

\bibliography{aaai2027}


\end{document}